\documentclass{article}

\PassOptionsToPackage{numbers,compress}{natbib}
\usepackage[preprint]{neurips_2025}

\usepackage{wrapfig}
\usepackage{microtype}
\usepackage{graphicx}
\usepackage{subcaption}
\usepackage{booktabs} 
\usepackage{tcolorbox}
\usepackage{hyperref}

\usepackage{amsmath}
\usepackage{amssymb}
\usepackage{mathtools}
\usepackage{amsthm}
\usepackage{algorithm}
\usepackage{algorithmic}
\usepackage{multirow}
\usepackage{booktabs}
\usepackage{colortbl} 
\usepackage{xcolor}
\usepackage{graphicx} 

\definecolor{arrowred}{RGB}{215, 60, 60} 
\definecolor{arrowblue}{RGB}{65, 105, 225}

\newcommand{\inc}[1]{\textcolor{arrowred}{\scriptsize~(\raisebox{1pt}{$\uparrow$}{#1})}}

\newcommand{\dec}[1]{\textcolor{arrowblue}{\scriptsize~(\raisebox{1pt}{$\downarrow$}{#1})}}

\usepackage[capitalize,noabbrev]{cleveref}

\theoremstyle{plain}
\newtheorem{theorem}{Theorem}[section]
\newtheorem{proposition}[theorem]{Proposition}
\newtheorem{lemma}[theorem]{Lemma}

\theoremstyle{definition}
\newtheorem{definition}[theorem]{Definition}
\newtheorem{assumption}[theorem]{Assumption}
\theoremstyle{remark}

\usepackage{bm} 

\definecolor{myblue}{RGB}{0,0,0}
\definecolor{myred}{RGB}{0,0,0}
\newcommand{\fix}[1]{\textcolor{myred}{#1}}

\definecolor{bggray}{RGB}{242, 242, 242}
\definecolor{maingray}{gray}{0.90}
\definecolor{lightblue}{RGB}{232, 242, 255}

\usepackage[textsize=tiny]{todonotes}
\usepackage{enumitem}

\title{Breaking the Martingale Curse: Multi-Agent Debate via Asymmetric Cognitive Potential Energy}

\author{
  Yuhan Liu\textsuperscript{1}\thanks{Equal contribution. The order was decided by a coin flip.},\quad
  Juntian Zhang\textsuperscript{2}\footnotemark[1],\\
  \textbf{Yichen Wu}\textsuperscript{3},\quad
  \textbf{Martin Takáč}\textsuperscript{1},\quad
  \textbf{Salem Lahlou}\textsuperscript{1},\quad
  \textbf{Xiuying Chen}\textsuperscript{1},\quad
  \textbf{Nils Lukas}\textsuperscript{1}\\
  \textsuperscript{1}Mohamed bin Zayed University of Artificial Intelligence \\ 
  \textsuperscript{2}Gaoling School of Artificial Intelligence, Renmin University of China \quad
  \textsuperscript{3}Harvard University\\
  \texttt{yuhan.liu@mbzuai.ac.ae} \quad
  \texttt{zhangjuntian@ruc.edu.cn}
}

\begin{document}

\maketitle

\begin{abstract}

Multi-Agent Debate (MAD) has emerged as a promising paradigm for enhancing large language model reasoning. However, recent work reveals a limitation:standard MAD cannot improve belief correctness beyond majority voting; we refer to this as the \textit{Martingale Curse}. This \textit{curse} arises because correlated errors cause agents to converge toward erroneous consensus, where debate merely reinforces collective mistakes rather than filtering noise.
We propose \textbf{AceMAD}, a framework that breaks the Martingale Curse by harnessing asymmetric cognitive potential energy to transform MAD from a random walk into a directed convergence process with positive drift. 
Through a peer-prediction mechanism, agents predict their peers' belief distributions, revealing asymmetric cognitive potential: truth-holders not only know the correct answer but also anticipate the crowd's misconceptions, while the hallucinating majority remains blind to their collective error. This asymmetry creates a potential energy gap that we quantify via strictly proper scoring rules. We prove this cognitive potential manifests as information-theoretic superiority and, under nonlinear aggregation, converts into submartingale drift toward truth, directly breaking the Martingale Curse. Experiments on challenging subsets across six benchmarks show AceMAD recovers sparse truth signals even when initial majorities are incorrect, substantially outperforming baseline methods.

\end{abstract}

\section{Introduction}
The advanced reasoning abilities of large language models (LLMs) have made collaborative reasoning via Multi-Agent Debate (MAD) a promising paradigm for solving complex tasks~\cite{li2025hiddenbench} including code generation~\cite{zhang2025thinking}, misinformation detection~\cite{liu2025truth}, mathematical problem solving~\cite{lei2024macm,chen2024comm} and visual QA~\cite{zhang2025weaving}.
The premise is intuitive: multiple agents exchange arguments iteratively, filtering noise through mutual critique to converge toward correct answers ~\cite{wangself,du2023improving}. However, recent work~\cite{choi2025debate} analysis reveals a fundamental barrier: without external supervision, standard MAD operates as a martingale process where expected belief correctness remains constant across debate rounds, reducing to majority voting in expectation. We refer to this as the \textit{\textbf{Martingale Curse}}.

While this characterization holds when agent errors are independent, we identify a critical oversight: in challenging reasoning tasks, LLMs exhibit correlated errors, systematic biases toward the same logical traps and misconceptions. Figure~\ref{fig:intro} illustrates this phenomenon: when the majority converges on a shared hallucination, for example choosing ``D: the ham" due to phonetic similarity, standard MAD amplifies rather than corrects the error. Our experiments in \S\ref{sec:empirical_analysis} (Table~\ref{tab:all_methods_comparison_n5}) confirm this empirically: on challenging subsets where initial majorities are wrong, majority voting achieves only 14.0\% accuracy, and standard MAD improves marginally to 22.1\%, far below what purely collaborative reasoning should achieve. This correlated noise regime is precisely where debate should prove indispensable, yet existing methods fail.

\begin{figure}[t]
    \centering
    \includegraphics[width=0.8\linewidth]{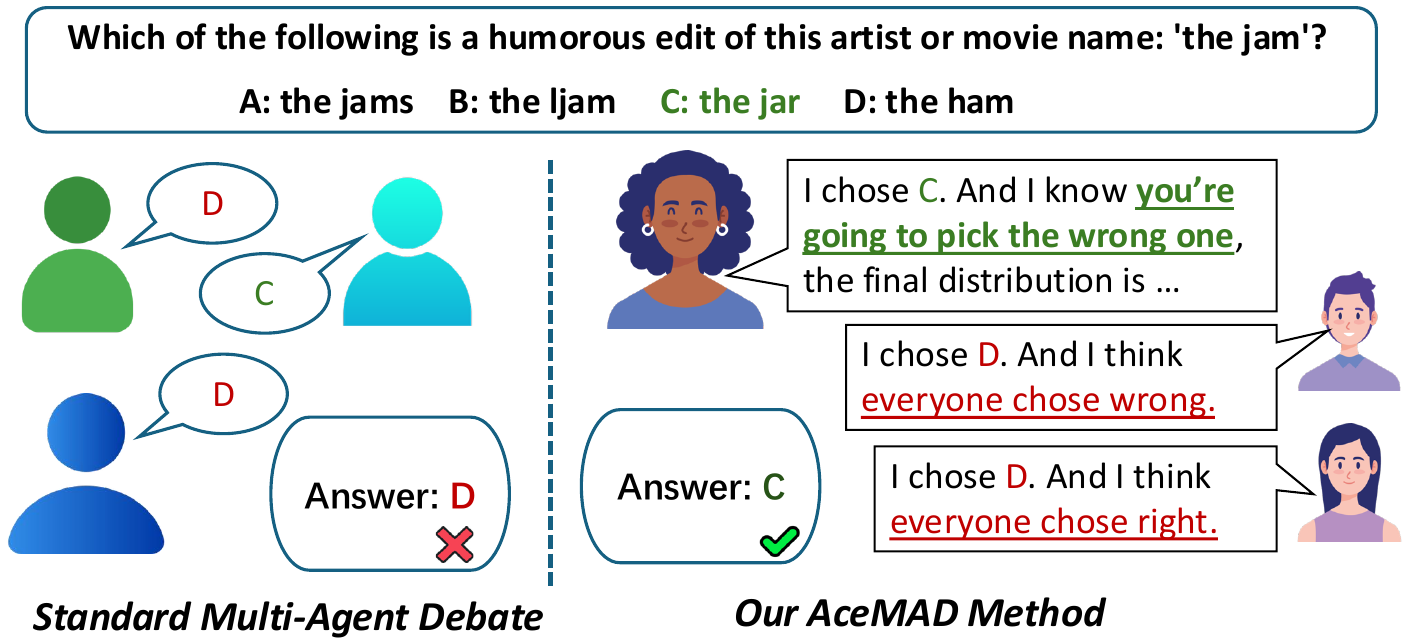}
    \caption{\textbf{Breaking Martingale Curse via Peer Prediction.} \textbf{(Left)} Standard MAD fails as the majority converges on a common misconception (``D"), drowning out the truth. \textbf{(Right)} \textit{AceMAD} recovers the truth (``C") by incentivizing agents to predict peer belief. The truth-holder is identified by correctly anticipating the crowd's behavior, allowing the system to filter out correlated noise.}
    \label{fig:intro}
\end{figure}
\textbf{Why does the \textit{Martingale Curse} persist?} Standard MAD treats all arguments as ``cheap talk"~\cite{farrell1996cheap}, updating beliefs through symmetric linear aggregation. Under correlated errors, this creates an echo chamber: the hallucinating majority reinforces each other's misconceptions, drowning isolated truth-holders in collective consensus. To break the curse, we need a mechanism that distinguishes truth-holders from the crowd without external labels.

We argue that \textit{truth-holders and the hallucinating majority possess \textbf{asymmetric cognitive potential energy}}. This asymmetry manifests in their second-order beliefs ~\cite{geanakoplos1992common}: (1) The majority, suffering from false consensus effect~\cite{ross1977false}, assumes everyone shares their view and is ``surprised" when contradicted; (2) truth-holders can anticipate the specific errors the majority will commit~\cite{nickerson1999we,eyster2010naive}. This cognitive gap between merely holding truth and understanding collective error creates a \textit{potential energy} difference that standard MAD fails to exploit.

To address this challenge, We propose the \textbf{AceMAD} (\textbf{A}symmetric \textbf{C}ognitive potential \textbf{E}nergy for \textbf{M}ulti-\textbf{A}gent \textbf{D}ebate), which converts this latent potential into directed drift toward truth. Through a peer-prediction mechanism~\cite{miller2005eliciting}, agents privately predict their peers' belief distributions before revealing arguments. We employ the Brier score, a strictly proper scoring rule to quantify prediction accuracy. Critically, truth-holders who correctly anticipate the crowd's hallucination achieve higher scores, while the majority incurs penalties for failed predictions. We leverage these scores through nonlinear weight amplification, transforming the zero-drift martingale into a submartingale with positive drift (Theorem~\ref{thm:submartingale}). This process elevates truth-holders' influence until they dominate the aggregate belief, when starting as the minority.

Our contributions are threefold as follows:

$\bullet$ \textit{\textbf{Algorithmic Protocol}}: We propose the \textbf{AceMAD} protocol, leveraging peer prediction and proper scoring to identify and amplify sparse truth signals in \fix{challenging interval}, requiring no external supervision.
    
$\bullet$ \textit{\textbf{Theoretical Analysis}}: We prove AceMAD achieves Blackwell dominance over standard MAD (Theorem~\ref{thm:blackwell}) by capturing second-order cognition, and demonstrate \textbf{how asymmetric cognitive potential energy converts into submartingale drift}, guaranteeing convergence to truth from minority positions (Theorem~\ref{thm:submartingale}).

$\bullet$ \textit{\textbf{Empirical Validation}}:  On challenging subsets across six benchmarks, AceMAD achieves 20.31\% over standard MAD. Ablations study confirm peer prediction is essential: removing second-order cognition collapses performance by 14.6\%, validating that cognitive potential energy, not first-order confidence, breaks the Martingale Curse.

\section{Preliminaries \fix{and Problem Formulation}}

In this section, we formalize  Standard Multi-Agent Debate, define the challenging interval, and present the formulation of ``Debate as a martingale process".

\subsection{Standard Multi-Agent Debate}

Consider a task with an input query $x \in \mathcal{X}$ and a discrete label space $\mathcal{Y}$. Let $\mathcal{A} = \{A_1, \dots, A_N\}$ be a set of $N$ LLM agents.
At any time step $t \ge 0$, each agent $A_i$ holds a private belief state, modeled as a probability distribution over the label space, denoted by $p_i^{(t)} \in \Delta(\mathcal{Y})$.

\paragraph{The Debate Protocol.}
The standard simultaneous debate protocol proceeds in rounds $t=1, \dots, T$. In each round, every agent $A_i$ first generates a natural language argument $m_i^{(t)}$ conditioned on the query $x$ and the accumulated conversation history $H^{(t-1)} = \{m_j^{(\tau)} : j \in [N], \tau < t\}$. Subsequently, agents update their private beliefs by incorporating the newly generated peer arguments. This update process is typically modeled as a linear combination of the agent's current belief and the visible peer beliefs, governed by the rule:
\begin{equation}
\label{rule}
    p_i^{(t+1)} = (1-\alpha) p_i^{(t)} + \alpha \sum_{j \neq i} \omega_{ij} p_j^{(t)}
\end{equation}
where $\alpha$ denotes the susceptibility to peer influence, and $\omega_{ij}$ represents the influence weight of agent $j$ on $i$. In standard settings, these weights are usually uniform ($\omega_{ij} = \frac{1}{N-1}$), treating all agents as equally credible sources. Finally, the system decision is derived from the aggregated belief $\bar{p}^{(T)} = \frac{1}{N} \sum_i p_i^{(T)}$ upon the conclusion of the debate.

\subsection{The Challenging Interval}

Standard aggregation methods such as majority voting rely on the \textit{Condorcet Jury Theorem}, which assumes agent errors are independent~\cite{Ladha1992TheCJ}. However, we argue that this assumption is violated in ``challenging'' reasoning tasks. We formally define the \textbf{Challenging Interval} as follows:

\begin{definition}[\textbf{Challenging Interval}]
\label{def:correlated_noise}
Let $y^* \in \mathcal{Y}$ denote the ground truth. The Challenging Interval is characterized by a collective failure pattern within a majority subset of agents $\mathcal{C} \subset \mathcal{A}$ (referred to as the ``Crowd'', where $|\mathcal{C}| > N/2$).
In this interval, \fix{the crowd exhibits systematic bias}, where the mode of the initial belief distribution $p_i^{(0)}$ converges to an incorrect answer $\hat{y} \neq y^*$, satisfying:
\begin{equation}
    \mathbb{P}(\arg\max p_i^{(0)} = \hat{y}) > 0.5, \quad \forall i \in \mathcal{C}.
\end{equation}
Furthermore, the errors among agents are positively correlated; that is, the likelihood of an agent $j$ failing increases conditioned on the failure of a peer $i$, formally defined as: 
\begin{equation}
        \mathbb{P}(\text{Error}_j \mid \text{Error}_i) > \mathbb{P}(\text{Error}_j), \quad \forall i, j \in \mathcal{C}.
    \end{equation}
Functionally, this definition captures scenarios dominated by ``common misconceptions,'' where agents do not make random errors but instead hallucinate consistently toward the same distractor.
\end{definition}

\subsection{Debate as a Martingale Process}

Recent work \citep{choi2025debate} characterizes the belief evolution in MAD as a martingale. We restate this property under our notation to motivate our proposed method.

\begin{proposition}[Martingale Property of Standard MAD]
\label{prop:martingale}
In a closed debate system without external oracle signals, if the belief update rule is linear (as in Eq.(\ref{rule})) and symmetric, the sequence of the average belief in the ground truth, denoted as $\mu_t = \bar{p}^{(t)}(y^*)$, forms a \textbf{martingale}:
\begin{equation}
    \mathbb{E}[\mu_{t+1} \mid H^{(t)}] = \mu_t.
\end{equation}
\end{proposition}
\textbf{Implication:} The expected correctness of the system is constant over time. Consequently, if the initial state is dominated by the Crowd where $\mu_0 < 0.5$ due to correlated noise, the debate will simply fluctuate around this erroneous mean and fail to converge to $y^*$. To yield better decisions, one must introduce a mechanism to induce a \textbf{positive drift}, effectively transforming the process into a sub-martingale satisfying $\mathbb{E}[\mu_{t+1}] > \mu_t$.
\section{When Does Debate Matter?}
\label{sec:empirical_analysis}

The ``Debate-as-Martingale" hypothesis \cite{choi2025debate} poses a fundamental challenge to the field: if MAD cannot theoretically outperform majority voting in expectation, then the computational cost of iterative interaction is unjustifiable. However, we contend that this theoretical characterization relies on strong assumptions: specifically, that initial errors are independent and that the aggregation is purely linear.

To investigate the validity of this hypothesis in challenging scenarios, we conducted a systematic empirical analysis comparing \textit{Majority Voting} against three distinct MAD (\textit{Standard/Decentralized}, \textit{Centralized}, and \textit{Sparse}) across the challenging subsets of six diverse benchmarks.

\begin{table*}[htbp]
\small
    \centering
    
    \caption{\textbf{Performance Comparison on Challenging Subsets ($N=5$).} We compare Majority Voting with \colorbox{bggray}{three MAD variants}. \textcolor{arrowred}{$\uparrow$} and \textcolor{arrowblue}{$\downarrow$} indicate relative gains and losses. Overall, MAD variants consistently outperform the Majority Voting baseline across the benchmarks.}
    \label{tab:all_methods_comparison_n5}

\resizebox{\textwidth}{!}{
\begin{tabular}{l cccccc | c}
\toprule
\textbf{Methods} 
& \textbf{TruthfulQA} 
& \textbf{ARC-C} 
& \textbf{BBH} 
& \textbf{LogiQA} 
& \textbf{MedQA} 
& \textbf{MMLU-Pro} 
& \textbf{Average} \\
\midrule
Majority Voting 
& 11.2 & 23.1 & 22.1 & 12.8 & 9.4 & 5.5 & 14.0 \\
\rowcolor{bggray}
Decentralized MAD 
& 15.7\inc{4.5} & \textbf{38.9}\inc{15.8} & 37.6\inc{15.5} & \textbf{13.1}\inc{0.3} & 21.0\inc{11.6} & 6.6\inc{1.1} & 22.1\inc{8.1} \\
\rowcolor{bggray}
Centralized MAD 
& 14.3\inc{3.1} & \textbf{38.9}\inc{15.8} & \textbf{40.3}\inc{18.2} & 11.3\dec{1.5} & \textbf{21.3}\inc{11.9} & 6.3\inc{0.8} & 22.1\inc{8.1} \\
\rowcolor{bggray}
Sparse MAD 
& \textbf{17.5}\inc{6.3} & 38.0\inc{14.9} & 38.3\inc{16.2} & 11.6\dec{1.2} & 20.3\inc{10.9} & \textbf{7.3}\inc{1.8} & \textbf{22.2}\inc{8.2} \\
\bottomrule
\end{tabular}
}
\end{table*}

\subsection{Experimental Setup}
\label{sec:setup_regime}

\paragraph{Benchmarks.} Unlike previous studies that evaluate performance on full test sets, where modern LLMs often achieve high accuracy ($>80\%$), we focus strictly on the Challenging Interval. We constructed Challenging subsets for six benchmark:(1) \textit{Hallucination}: \textbf{TruthfulQA}~\cite{lin2022truthfulqa};
(2) \textit{Reasoning}: \textbf{ARC-C}~\cite{clark2018think}, \textbf{BBH}~\cite{suzgun2023challenging}; (3) \textit{Domain Knowledge}: \textbf{LogiQA}~\cite{liu2021logiqa}, \textbf{MedQA}~\cite{jin2021disease}, and \textbf{MMLU-Pro}~\cite{wang2024mmlu}. \fix{These subsets are defined by instances where a naive single-agent baseline consistently fails, often due to logical traps or domain-specific misconceptions.} For fairness in comparison,all baselines are evaluated on the same data subsets.More details are provided in Appendix~\ref{app:experiment}.

\paragraph{Baselines.}
We evaluate performance across varying agent group sizes $N \in \{2, 3, 5\}$ using four distinct protocols. As a baseline, we employ \textbf{Majority Voting}, representing ensemble aggregation without inter-agent communication. We compare this against \textbf{Standard MAD (Decentralized)}, where agents engage in peer-to-peer argumentation, and \textbf{Centralized MAD}, which organizes the debate under a board-meeting structure. Finally, we introduce \textbf{Sparse MAD}, a specialized debate variant explicitly designed for sparse signal recovery.For all the MAD method, we adopt agent
group sizes $N = 5$ in our main comparison and will ablate on the effect of $N$. For single-agent baselines, we average across 5 independent runs.

\subsection{Key Empirical Observations}
\label{sec:results_debate_vs_voting}
\begin{figure}[htbp]
    \centering
    \includegraphics[width=\linewidth]{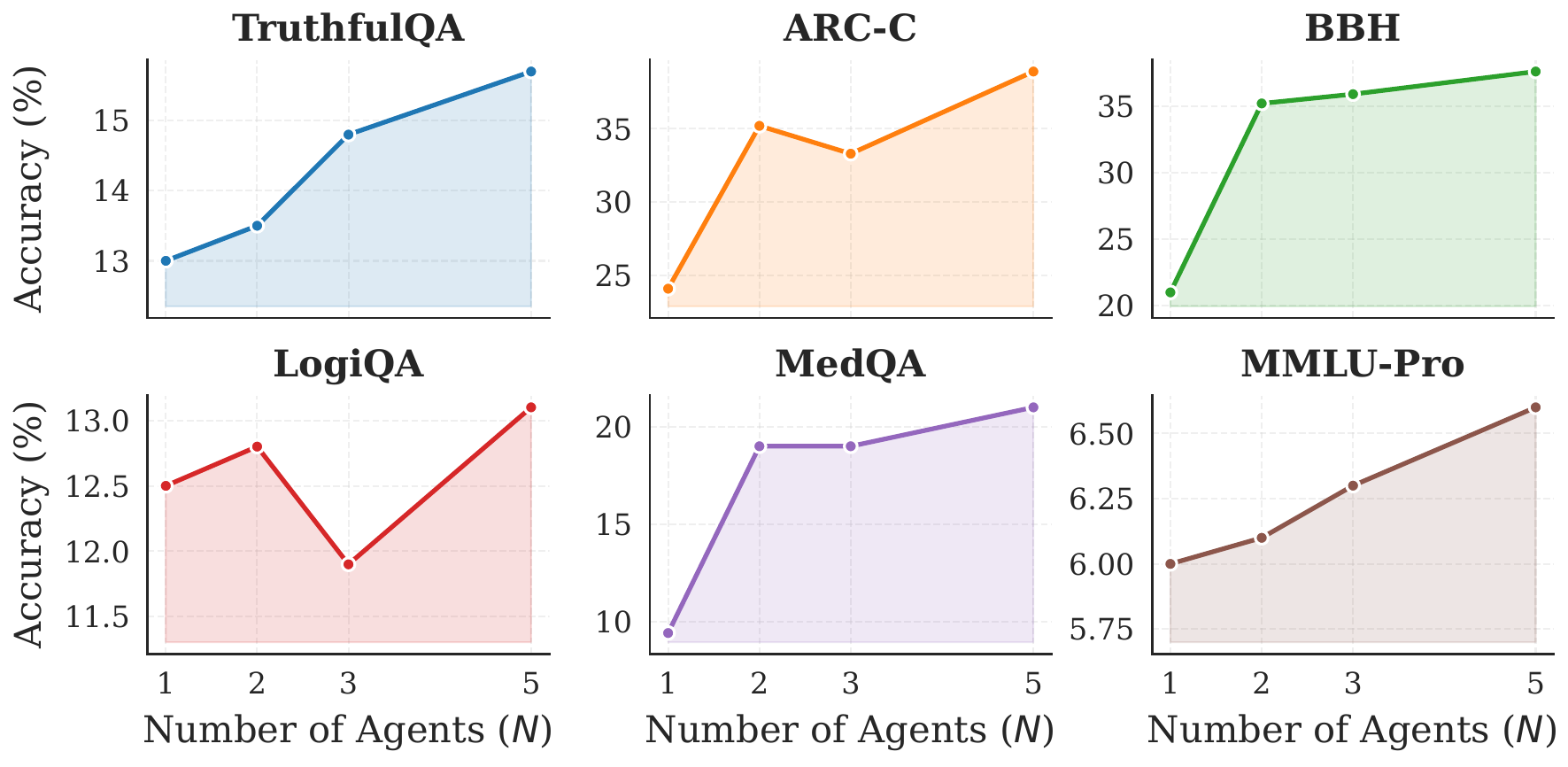}
    \caption{\textbf{Impact of Agent Scaling.} Performance on challenging subsets improves with group size ($N=1 \to 5$), indicating that multi-agent debate effectively filters correlated noise.}
    \label{fig:ablation}
\end{figure}
Table~\ref{tab:all_methods_comparison_n5} and Figure~\ref{fig:ablation} present the comprehensive results comparing \textit{Majority Voting} against three MAD methods across varying agent counts ($N \in \{1, 2, 3, 5\}$). Our analysis reveals three critical behavioral divergences.

\textbf{1. Voting Stagnates in Correlated Noise.}
As detailed in Table~\ref{tab:all_methods_comparison_n5}, Majority Voting struggles to extract correct answers in the challenging interval. Across all datasets, the voting baseline shows negligible improvement, or even degradation, as the ensemble size increases. For instance, increasing $N$ from 2 to 5 fails to break the performance ceiling on tasks like TruthfulQA and LogiQA. This empirical plateau confirms that errors in these adversarial regimes are highly correlated; simply adding more static agents merely reinforces the false consensus rather than diluting it.

\textbf{2. Debate Consistently Outperforms Voting.}
Table~\ref{tab:all_methods_comparison_n5} demonstrates that this limitation is structural to voting, not the models. \textbf{All MAD variants consistently outperform Majority Voting} across the tested benchmarks. In the most challenging reasoning tasks such as ARC-C and BBH, debate mechanisms achieve substantial margins over the voting baseline. This indicates that the debate process itself, regardless of the specific topology, acts as a necessary ``denoising" filter that disrupts the homogeneous error patterns of the crowd.

\textbf{3. Positive Scaling Dynamics.}
While ``scaling laws" typically unfold over orders of magnitude, we observe a distinct \textbf{positive scaling trend} even within the limited range of $N=1$ to $5$. As illustrated in Figure~\ref{fig:ablation}, unlike the flat voting trajectory, the performance of Decentralized MAD improves monotonically with the number of agents. 
Adding agents to a debate introduces new reasoning paths that help the system escape local optima, a property that is essential for breaking the Martingale Curse.

\subsection{Debate is Essential in the Challenging Interval}

The empirical evidence decisively refutes the generalized claim that debate is unnecessary. While prior studies suggest voting suffices for simple tasks, our findings demonstrate that \textbf{debate is indispensable specifically in the challenging interval}. In regimes dominated by correlated noise where Majority Voting is trapped by false consensus, debate mechanisms such as centralized, decentralized, or sparse variants act as a critical filter to leverage argument heterogeneity and identify correct reasoning paths. However, although debate significantly outperforms voting, the absolute performance gains suggest that standard MAD alone is insufficient, motivating our proposed AceMAD framework to further amplify these sparse truth signals.

\section{Theoretical Framework}
\label{sec:theory}
Having established that standard MAD under linear aggregation degenerates into a martingale in the challenging interval (Proposition~\ref{prop:martingale}), we now present the theoretical foundation of \textbf{AceMAD}. We formally demonstrate how our protocol breaks the Martingale Curse by (1) revealing asymmetric cognitive potential energy through enhanced observables,  (2)  statistically quantifying this potential energy via proper scoring, and (3) converting potential energy into submartingale drift toward truth through nonlinear amplification. 
Detailed proofs for all theorems are provided in Appendix~\ref{app:theory}.

\subsection{Asymmetric Cognitive Potential Energy in MAD}

To formalize the cognitive asymmetry, we characterize the observable signals available in different debate protocols.

\begin{definition}[\textbf{Debate Observable Space}]
\label{def:channels}
Let $\mathcal{M}$ be the space of natural language arguments and $\Delta(\mathcal{Y})$ be the belief simplex. We define the observable transcript:

The \textbf{Standard MAD} outputs a transcript consisting solely of arguments and self-reported beliefs:
\begin{equation}
Z_{\text{std}} = \{(m_i^{(t)}, p_i^{(t)})\}_{i=1}^N .
\end{equation}

The \textbf{AceMAD} outputs an augmented transcript that includes peer prediction and realized scores:
\begin{equation}
    Z_{\text{info}} = \{(m_i^{(t)}, p_i^{(t)}, \hat{q}_i^{(t)}, S_i^{(t)})\}_{i=1}^N ,
\end{equation}
where $\hat{q}_i^{(t)} \in \Delta(\mathcal{Y})$ represents agent $i$'s prediction of the peer belief distribution, and $S_i^{(t)}$ is the realized score.
\end{definition}

\paragraph{Interpretation.}
The key distinction lies in the \emph{second-order beliefs} $\{\hat{q}_i\}$. Standard MAD observes only first-order beliefs (what agents think is true), while \textbf{AceMAD} additionally captures \emph{meta-cognitive signals} (what agents think others believe). This distinction is not merely informational, it reveals \emph{asymmetric cognitive potential energy}. Truth-holders who understand collective error possess fundamentally different second-order models than a hallucinating majority that assumes consensus.

\begin{theorem}[\textbf{Information-Theoretic Manifestation of Cognitive Potential}]
\label{thm:blackwell}
Let $\sigma_{std}$ and $\sigma_{info}$ be the information channels corresponding to standard MAD and AceMAD, respectively. There exists a stochastic kernel $\kappa$ such that $\sigma_{std} = \kappa \circ \sigma_{info}$. Consequently, $\sigma_{info} \succeq_{Blackwell} \sigma_{std}$.
\end{theorem}

\textbf{Insight:}  Information-theoretically, AceMAD provides a strictly richer signal space. Decision-makers observing $Z_{info}$ can always reconstruct $Z_{std}$ (by discarding forecasts), but the reverse is impossible. This extra signal is crucial for identifying truth-holders.

\subsection{Quantifying Cognitive Potential via Proper Scoring}

In the challenging interval, the majority is confidently wrong. To distinguish truth-holders, we must quantify the \emph{potential energy gap} between agents who merely hold correct beliefs and those who understand why others error.

\begin{definition}[\textbf{Peer Prediction Scoring}]
\label{def:scoring}
The realized average belief of agent $i$'s peers is defined as:
\begin{equation}
\label{score}
 \bar{Q}_{-i}^{(t)} = \frac{1}{N-1}\sum_{j \neq i} p_j^{(t)}.
\end{equation}
We evaluate agent $i$'s prediction $\hat{q}_i^{(t)}$ using the Brier Score function $S: \Delta(\mathcal{Y}) \times \Delta(\mathcal{Y}) \to [0,1]$:
\begin{equation}
    S_i^{(t)} = 1 - ||\hat{q}_i^{(t)} - \bar{Q}_{-i}^{(t)}||_2^2.
\end{equation}
\end{definition}

\begin{theorem}[\textbf{Cognitive Potential Energy Separation}]
\label{thm:score_separation}
Consider the challenging interval where the Crowd $\mathcal{C}$ holds a shared misconception and the truth-holder $E$ holds the truth. Assume $E$ possesses a correct second-order model of the crowd's error. Then, according to the scoring rule in Definition~\ref{def:scoring}, the expected score of the truth-holder is strictly greater than that of any crowd agent $c \in \mathcal{C}$:
\begin{equation}
    \mathbb{E}[S_E] > \mathbb{E}[S_c].
\end{equation}
\end{theorem}

\textbf{Intuition:} The score separation stems from ``surprise." The Crowd, suffering from the False Consensus~\cite{ross1977false}, erroneously predicts that everyone agrees with them. When the truth-holder dissents, the Crowd incurs a prediction penalty. Conversely, the truth-holder anticipates the Crowd's hallucination, minimizing surprise, and maximizing their score.This score differential is the quantitative signature of \textit{asymmetric cognitive potential energy:} it reveals who possesses superior meta-cognitive understanding of the system's epistemic state.

\subsection{Converting Potential into Submartingale Drift}

The standard MAD uses linear averaging, which preserves the martingale property. To break this curse, we must release the latent cognitive potential energy identified by scoring, converting it into directed motion toward truth.

\begin{definition}[\textbf{Nonlinear Potential-to-Influence Conversion}]
\label{def:update_rule}
Let $w_i^{(t)}$ be the influence weight of agent $i$ at round $t$. We define the Multiplicative Weight Update rule as:
\begin{equation}
    w_i^{(t+1)} = w_i^{(t)} \cdot \exp(\eta \cdot S_i^{(t)}),
\end{equation}
where $\eta > 0$ is the amplification rate parameter. The aggregated belief is then $p_w^{(t)} \propto \sum_i w_i^{(t)} p_i^{(t)}$.
\end{definition}

\begin{theorem}[\textbf{Submartingale Convergence: Breaking the Curse}]
\label{thm:submartingale}
Let $\mu_t = p_w^{(t)}(y^*)$ be the probability assigned to the ground truth by the weighted aggregate belief. Under the conditions of Theorem \ref{thm:score_separation} and using the update rule in Definition \ref{def:update_rule}, for sufficiently small $\eta > 0$, the process $\{\mu_t\}_{t \ge 0}$ forms a \textbf{sub-martingale}:
\begin{equation}
    \mathbb{E}[\mu_{t+1} \mid H^{(t)}] \ge \mu_t + \Delta(\eta),
\end{equation}
where $\Delta(\eta) > 0$ represents a positive drift toward the truth.
\end{theorem}

\textbf{Breaking the Martingal Curse:} Unlike the random walk of standard MAD, it guarantees that the system's belief in the truth monotonically increases in expectation. 
The mechanism is clear: (1) Theorem~\ref{thm:score_separation} ensures truth-holders consistently achieve higher scores due to their superior cognitive potential; (2) Exponential weight accumulation  ($w_E \propto e^{\eta \sum S_E}$) amplifies this advantage over rounds; (3) Eventually, the truth-holder's influence dominates the aggregate, pulling the system toward the correct answer even from a minority starting position.
\section{The AceMAD Protocol}
\label{sec:algorithm}

\begin{figure}[t]
    \centering
    \includegraphics[width=0.85\linewidth]{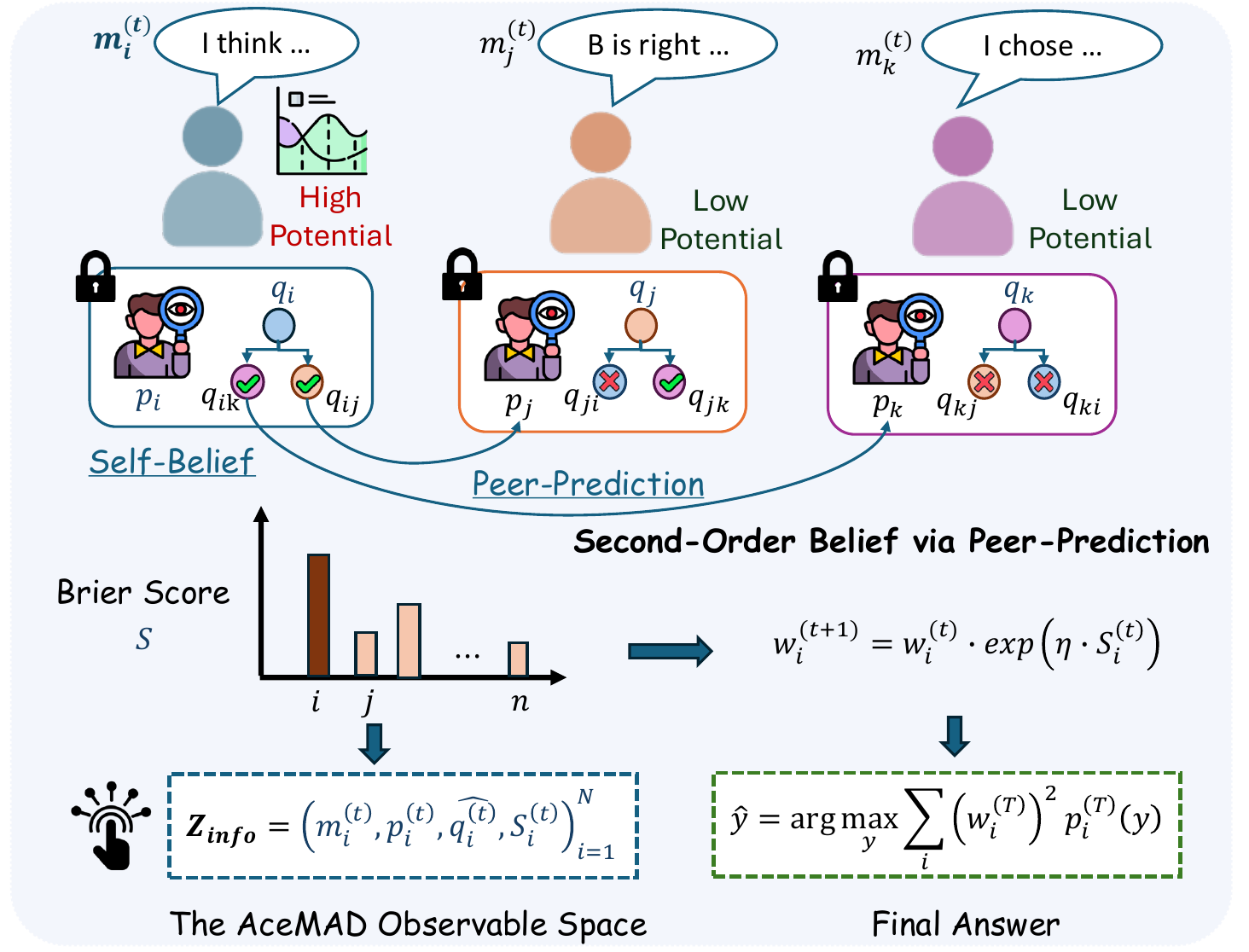}
    \caption{\textbf{The AceMAD: Converting Asymmetric Cognitive Potential into Submartingale Drift.} The process iterates through four phases: (1) \textbf{Argumentation:} Agents exchange arguments ($m_i^{(t)}$) in a shared context; (2) \textbf{Signal Extraction:} Agents privately commit to their self-belief ($p_i$) and peer-prediction ($q_i$),revealing asymmetric second-order cognition; (3) \textbf{Verification:} Brier Score $S_i$ quantifies the cognitive potential gap; (4) \textbf{Non-linear Amplification:} Exponential weight updates convert potential into directional drift, progressively amplifying truth-holders' influence until they dominate the aggregate, breaking the Martingale Curse.}
    \label{fig:method}
\end{figure}

Based on the theoretical framework established in Section \ref{sec:theory},we present the concrete realization of asymmetric cognitive potential energy in the \textbf{AceMAD}. The algorithm operationalizes our theory through a four-phase cycle that extracts, quantifies, and amplifies the latent cognitive asymmetry.

\subsection{Protocol Design}
\begin{table*}[t]
    \centering
    \caption{Main results on challenging subsets ($N=5$). We evaluate AceMAD ($T \in \{2, 3, 5\}$) against different baselines across six benchmarks using both GPT-4o-mini and open-source models. Performance gains over the best baseline are shown in parentheses.}
    \label{tab:main_results}

\resizebox{\textwidth}{!}{
\begin{tabular}{l cccccc | c}
\toprule
\textbf{Methods} 
& \textbf{ARC-C} 
& \textbf{LogiQA} 
& \textbf{MMLU-Pro} 
& \textbf{TruthfulQA} 
& \textbf{MedQA} 
& \textbf{BBH} 
& \textbf{Average} \\
\midrule

\rowcolor{maingray}
\multicolumn{8}{c}{\textit{Closed-Source Model: GPT-4o-mini}} \\

Single Agent
& 20.37 & 16.25 & 6.42 & 31.88 & 11.61 & 33.79 & 20.05 \\

\textit{Majority Voting}
& 23.15 & 21.88 & 5.50 & 17.94 & 23.23 & 22.07 & 18.96 \\

\textit{Decentralized MAD}
& 41.67 & 22.19 & 6.88 & 33.63 & 23.23 & 44.14 & 28.62 \\

\textit{Centralized MAD}
& 41.67 & 22.19 & \underline{\textbf{8.72}} & \underline{39.51} & \underline{24.19} & 39.66 & 29.32 \\

\textit{Sparse MAD}
& \underline{43.52} & \underline{22.19} & \underline{\textbf{8.72}} & 32.74 & 23.23 & \underline{47.24} & \underline{29.61} \\

\rowcolor{lightblue}
\textit{\textbf{AceMAD}} (T=2)
& 56.48 \inc{12.96} & 20.00 \dec{2.19} & 8.10 \dec{0.62} & 37.67 \dec{1.84} & \textbf{39.03} \inc{14.84} & 72.76 \inc{25.52} & 39.01 \inc{9.40} \\

\rowcolor{lightblue}
\textit{\textbf{AceMAD}} (T=3)
& 56.48 \inc{12.96} & \textbf{37.74} \inc{15.55} & \textbf{8.72} \inc{0.00} & \textbf{39.91} \inc{0.40} & 38.39 \inc{14.20} & \textbf{78.28} \inc{31.04} & \textbf{49.92} \inc{20.31} \\

\rowcolor{lightblue}
\textit{\textbf{AceMAD}} (T=5)
& \textbf{59.26} \inc{15.74} & 20.63 \dec{1.56} & 7.95 \dec{0.77} & 40.81 \inc{1.30} & \textbf{39.03} \inc{14.84} & 77.24 \inc{30.00} & 40.82 \inc{11.21} \\

\midrule

\rowcolor{maingray}
\multicolumn{8}{c}{\textit{Open-Source Model: Qwen3-235B-A22B-Instruct}} \\

Single Agent
& 65.74 & 42.50 & 35.93 & 74.89 & 55.81 & 42.07 & 52.82 \\

\textit{Majority Voting}
& 69.44 & 43.12 & 14.07 & 62.78 & 56.45 & 13.79 & 43.28 \\

\textit{Decentralized MAD}
& \underline{49.07} & 45.62 & \underline{28.90} & 71.75 & 59.03 & 30.69 & \underline{47.51} \\

\textit{Centralized MAD}
& 13.89 & 45.62 & 24.01 & 71.75 & 54.84 & \underline{65.17} & 45.88 \\

\textit{Sparse MAD}
& 13.89 & \underline{47.81} & 13.91 & \underline{\textbf{83.41}} & \underline{60.97} & 63.10 & 47.18 \\

\rowcolor{lightblue}
\textbf{\textit{AceMAD}} (T=2)
& \textbf{82.43} \inc{33.36} & \textbf{59.69} \inc{11.88} & 14.37 \dec{14.53} & 60.49 \dec{22.92} & \textbf{69.68} \inc{8.71} & 89.66 \inc{24.49} & \textbf{62.72} \inc{15.21} \\

\rowcolor{lightblue}
\textbf{\textit{AceMAD }}(T=3)
& 80.56 \inc{31.49} & \textbf{59.69} \inc{11.88} & 27.06 \dec{1.84} & 78.03 \dec{5.38} & 27.10 \dec{33.87} & 91.03 \inc{25.86} & 60.58 \inc{13.07} \\

\rowcolor{lightblue}
\textit{\textbf{AceMAD}} (T=5) 
& 77.78 \inc{28.71} & 28.44 \dec{19.37} & \textbf{36.54} \inc{7.64} & 40.81 \dec{42.60} & 42.26 \dec{18.71} & \textbf{92.07} \inc{26.90} & 52.98 \inc{5.47} \\

\bottomrule
\end{tabular}
}
\end{table*}

The algorithm operates over $T$ discrete rounds. We initialize all agents with uniform weights $w_i^{(0)} = 1$. Each round $t \in [1, T]$ consists of four synchronous phases:

\textbf{Phase 1: Argumentation (Context Transmission).}
Each agent $i$ generates a natural language argument $m_i^{(t)}$ conditioned on the query $x$ and the conversation history $H^{(t-1)}$. This phase is identical to Standard MAD, ensuring equal information access.

\textbf{Phase 2: Signal Extraction (Prediction).}
Before observing current peer arguments, each agent privately submits: 
(1) \textbf{Self-Belief $p_i^{(t)}$}: The agent's prediction of the ground truth $Y$;
(2) \textbf{Peer-Prediction $\hat{q}_i^{(t)}$}: A prediction of the \textit{average belief distribution} of all other agents in the current round.

\textbf{Phase 3: Verification (Scoring).}
The system calculates the realized peer average $\bar{Q}_{-i}^{(t)}$ Eq.(\ref{score}) and assigns each agent a reputation score $S_i^{(t)}$ using the Brier score, measuring alignment between prediction $\hat{q}_i^{(t)}$ and realization $\bar{Q}_{-i}^{(t)}$.

\textbf{Phase 4: Non-linear Amplification (Update).}
To break the martingale curse, influence weights are updated via the multiplicative rule derived in Theorem \ref{thm:submartingale}
, exponentially amplifying agents who successfully predict group behavior. The final decision uses weighted aggregation:
\begin{equation}
\label{prediction}
 \hat{y} = \arg\max_y \sum_i (w_i^{(T)})^2 p_i^{(T)}(y).   
\end{equation}
The complete procedure is summarized in Algorithm \ref{alg:AceMAD}.

\subsection{Inducing Agent Heterogeneity}
\label{sec:heterogeneity_principle}

To satisfy Definition~\ref{def:correlated_noise}'s requirement of a hallucinating majority and isolated truth-holders, we induce \textbf{inference heterogeneity} within a single base LLM through a dual-strategy approach:

(1)\textbf{The Crowd:} A subset of agents is configured to approximate the model's maximum likelihood behavior. By suppressing stochasticity, these agents consistently converge on the high-probability tokens present in the pre-training distribution, thereby reproducing the ``common misconceptions" and ``false consensus" characteristic of correlated noise;(2)\textbf{The Truth-Holder}: 
A minority of agents is intentionally perturbed from the standard generation path by assigning distinct personas or eliciting alternative thinking modes (e.g., System~2 reasoning~\cite{kahneman2011thinking,li2025system}), together with increased sampling stochasticity. This induces exploration of the distributional tail and enables recovery of sparse truth signals otherwise suppressed by majority bias.
This operationalizes the theoretical setup: the majority exhibits correlated errors, while truth-holders provide the asymmetric cognitive potential energy necessary.

\section{Experiments}

\subsection{Experimental Setup}
We evaluate AceMAD on the same challenging subsets and baseline methods($N=5$) established in Section~\ref{sec:setup_regime}.
The key difference lies in how we satisfy the requirement of inference heterogeneity: we induce a dual-strategy agent distribution within the same base LLM by configuring a majority as ``Crowd" agents (reproducing common misconceptions) and a minority as ``Truth-holder" agents (assigned critical-thinking personas with higher sampling stochasticity to explore sparse signals). We test \textit{AceMAD} across $T \in \{2, 3, 5\}$ rounds against Decentralized, Centralized, and Sparse MAD (fixed at $T=3$). Implementation details and prompt are provided in Appendices~\ref{app:experiment} and \ref{app:prompt}.

\subsection{Main Results}

Table~\ref{tab:main_results} summarizes our main results across six benchmarks. \textit{AceMAD} consistently and substantially outperforms all baseline methods, empirically validating our theoretical claim that peer prediction breaks the Martingale Curse. With \texttt{GPT-4o-mini}, \textit{AceMAD} at  $T=3$ rounds achieves 49.92\% average accuracy, representing a 74\% relative improvement over Decentralized MAD and a huge improvement over Majority Voting, which fails to filter correlated noise as predicted. The gains are most pronounced on \textbf{reasoning-heavy} tasks such as BBH and ARC-C, where logical traps induce the strongest collective hallucinations.Importantly, these improvements generalize across model architectures: similar performance gains are observed on the open-source model \texttt{Qwen3-235B-A22B-Instruct}, confirming that our peer-prediction mechanism targets a fundamental cognitive asymmetry rather than exploiting model-specific biases. To further verify robustness, we evaluate AceMAD on the \texttt{DeepSeek-V3.1} model and \texttt{Llama-3.1-8B-Instruct}, observing consistent and substantial gains over all MAD baselines (see Appendices~\ref{app:deepseek} and \ref{app:8b_analysis} for detailed results). Regarding iteration depth, $T=3$ emerges as the optimal configuration: $T=2$ provides insufficient rounds for the sub-martingale drift to amplify truth signals, while $T=5$ exhibits performance degradation on tasks like LogiQA and MedQA, suggesting that excessive iterations allow the persistent majority to reassert pressure, causing truth-holders to eventually capitulate. These patterns confirm that \textit{AceMAD} successfully induces the positive drift toward truth guaranteed by Theorem~\ref{thm:submartingale}, effectively recovering truth when initial majorities are confidently wrong.

\section{Discussion}

We now discuss the mechanism to understand \textit{when and why asymmetric cognitive potential energy emerges}, and \textit{identify its fundamental boundaries}.

\noindent\textbf{RQ1: What forms of heterogeneity are necessary? }

Table~\ref{tab:heterogeneity} reveals that AceMAD's success is agnostic to the source of diversity which is defined in \S\ref{sec:heterogeneity_principle}.We test three distinct strategies: \textit{cross-model mixing} by combining \texttt{GPT-4o-mini} and \texttt{Llama-3.1-8B-Instruct} to leverage divergent pre-training biases, \textit{persona-driven roles} by assigning specialized characters to prevent premature consensus, and \textit{cognitive system variation} by mixing ``fast thinking" with ``slow thinking" patterns.
Table~\ref{tab:heterogeneity} reveals that all three substantially outperform homogeneous Decentralized MAD, with task-specific sensitivity: cross-model heterogeneity excels on knowledge tasks like MedQA by pooling varied internal representations, while persona and cognitive diversity prove most effective on logical reasoning benchmarks by enforcing multi-perspective scrutiny. It shows that the mechanism is the asymmetry in second-order beliefs, not the particular instantiation of diversity.

\begin{table*}[htbp]
\centering

\caption{Impact of heterogeneity configurations on \textit{AceMAD} performance. Results demonstrate the effectiveness of inducing diversity via cross model mixing, persona-driven role-playing, and cognitive system variation to satisfy.}
\label{tab:heterogeneity}

\resizebox{\textwidth}{!}{
\begin{tabular}{l cccccc | c}
\toprule
\textbf{Methods} 
& \textbf{ARC-C} 
& \textbf{LogiQA} 
& \textbf{MMLU-Pro} 
& \textbf{TruthfulQA} 
& \textbf{MedQA} 
& \textbf{BBH} 
& \textbf{Average} \\
\midrule

\textit{Decentralized MAD}
& \underline{41.67} & \underline{22.19} & \underline{8.72} & \underline{33.63} & \underline{23.23} & \underline{44.14} & \underline{28.93} \\

\rowcolor{lightblue}
\textit{\textbf{AceMAD}} (Different LLMs)
& 56.48 \inc{14.81} & 20.94 \dec{1.25} & 9.63 \inc{0.91} & 37.22 \inc{3.59} & \textbf{38.71} \inc{15.48} & 74.14 \inc{30.00} & 39.52 \inc{10.59} \\

\rowcolor{lightblue}
\textit{\textbf{AceMAD}} (Different Persona)
& 60.19 \inc{18.52} & \textbf{23.13} \inc{0.94} & \textbf{12.08} \inc{3.36} & \textbf{46.19} \inc{12.56} & 23.23 \inc{0.00} & \textbf{76.21} \inc{32.07} & \textbf{40.17} \inc{11.24} \\

\rowcolor{lightblue}
\textit{\textbf{AceMAD}} (Different Thinking System)
& \textbf{61.11} \inc{19.44} & 22.19 \inc{0.00} & \textbf{12.08} \inc{3.36} & 43.95 \inc{10.32} & 23.23 \inc{0.00} & 37.59 \dec{6.55} & 33.36 \inc{4.43} \\

\bottomrule
\end{tabular}
}
\end{table*}

\noindent\textbf{RQ2: What is the role of second-order peer prediction? } 

To isolate the impact of our core components, we compare the \textit{Full Protocol} against two ablation variants in Figure~\ref{fig:ablation}: (1) \textit{\textbf{Self-Belief Only}}, which modulates weights using only the agent's self-reported confidence ($p_i$) without peer forecasts; and (2) \textit{\textbf{Uniform Weight}}, effectively reducing the system to \textit{Standard MAD}. The results consistently show that the \textit{Full Protocol} achieves the highest accuracy across all benchmarks, particularly in tasks like ARC-C and BBH where the margin over the \textit{Self-Belief Only} variant is most pronounced. This gap demonstrates that first-order belief is a noisy signal in the \textit{Challenging Interval}, as hallucinating majorities are often highly confident yet fundamentally wrong. Only by incorporating the second-order \textit{Peer-Prediction} ($\hat{q}_i$) can the system identify  agents who not only know the truth but also correctly model the crowd's misconception. These findings confirm that peer prediction is the critical mechanism for breaking the \textit{Martingale Curse} and inducing a positive drift toward the ground truth.

\begin{figure}[htbp]
    \centering
    \includegraphics[width=1\linewidth]{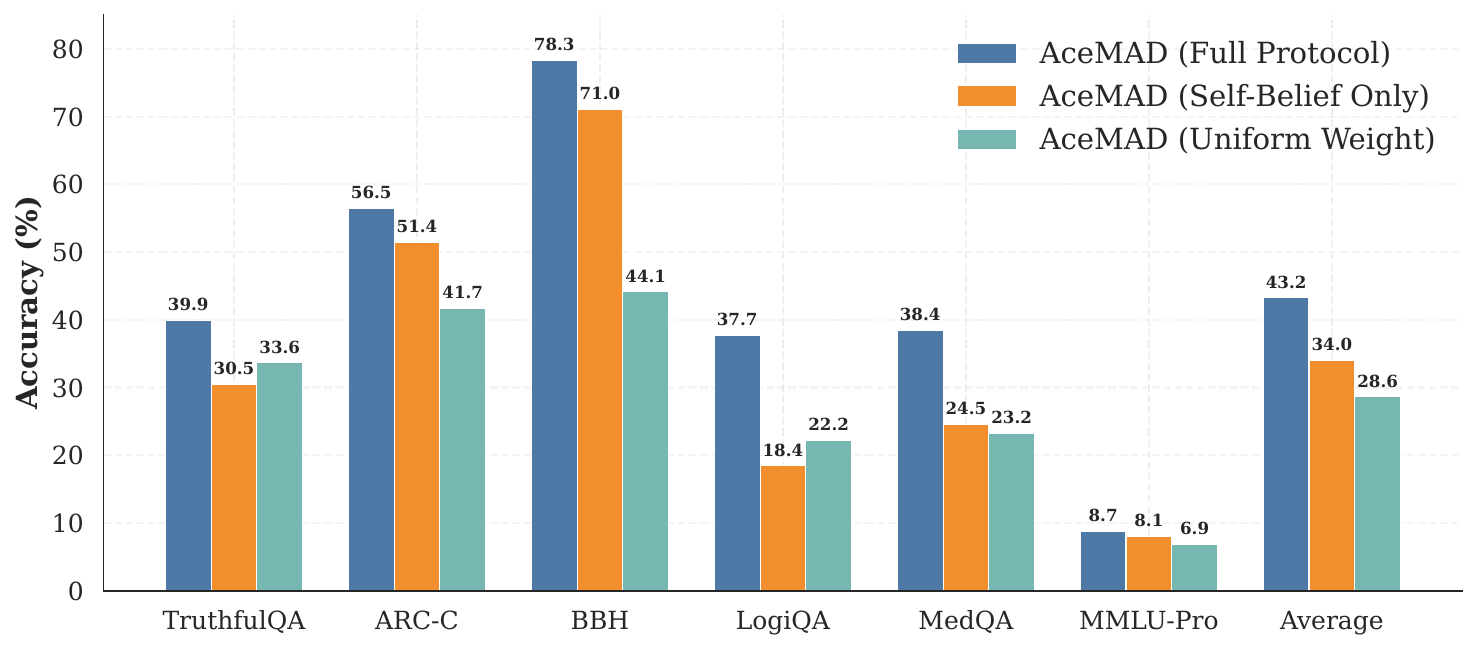}
    \caption{Ablation study of \textit{AceMAD}. We compare the full protocol against two variants: (1) \textit{Self-Belief Only}, which ignores second-order peer forecasts, and (2) \textit{Standard MAD}, which uses uniform weighting. The results highlight that peer prediction is essential for filtering out incorrect majority beliefs.}
    \label{fig:ablation}
\end{figure}

\noindent\textbf{RQ3: Does scaling always help? } 

Figure~\ref{fig:agent_num} uncovers a striking non-monotonicity. Performance improves substantially from $N=1$ to $10$, confirming that larger groups increase the probability of surfacing at least one truth-holder and provide richer reasoning paths for the peer-prediction mechanism to exploit. However, beyond $N=10$, accuracy plateaus or even declines, most notably on BBH. This reveals a critical boundary: as the crowd grows, correlated noise becomes increasingly homogeneous and overwhelming. Even with exponential weight amplification, the truth-holder's signal can be drowned out by sheer majority pressure. Moreover, agent density creates information overload in the natural language channel, degrading signal-to-noise ratios and enabling the crowd to converge toward sophisticated but incorrect consensus. This suggests an optimal group size exists where information gain is maximized while minimizing collective reinforcement of misconceptions.

\begin{figure}[htbp]
    \centering
    \includegraphics[width=1\linewidth]{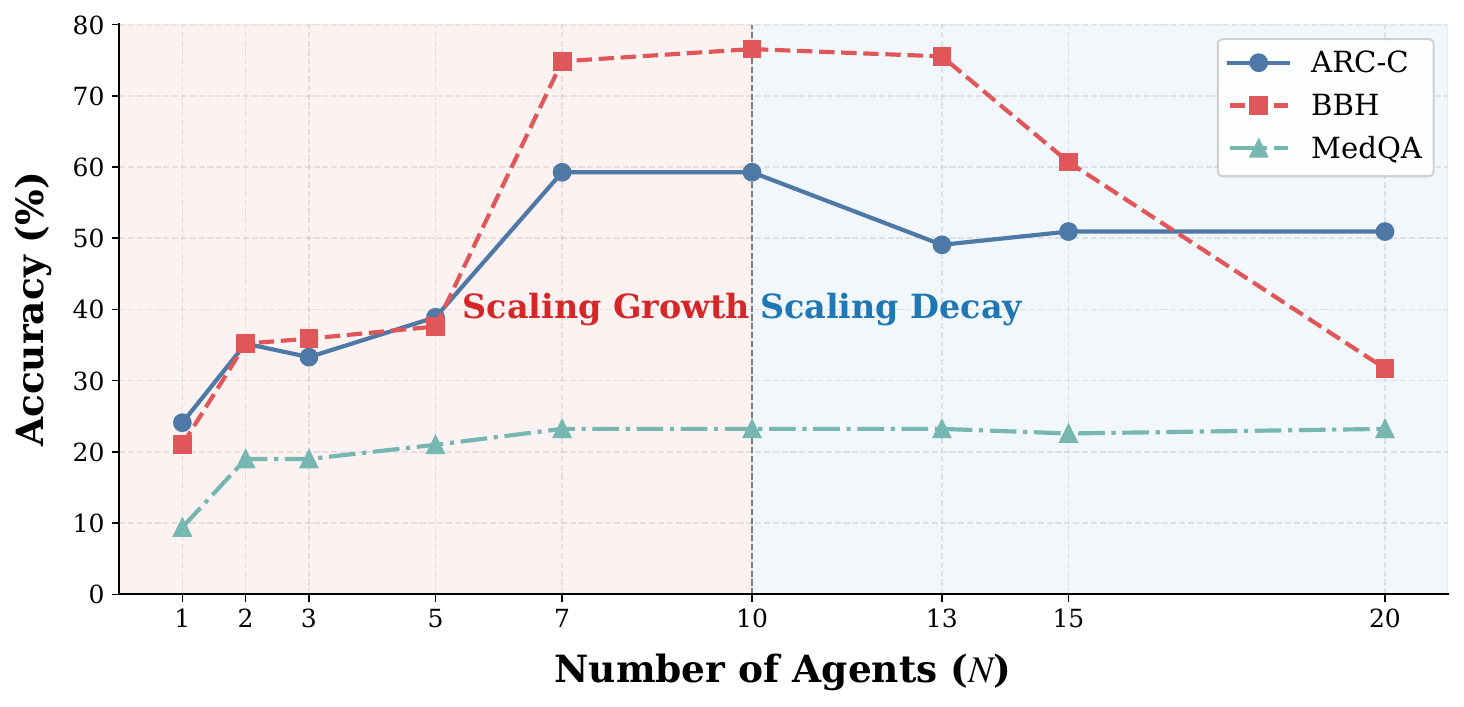}
    \caption{Impact of agent scaling on accuracy. The performance trajectory exhibits two distinct regimes: (1) \textbf{Scaling Growth} ($N \le 10$), where increasing the group size enhances the probability of surfacing sparse truth signals, and (2) \textbf{Scaling Decay} ($N > 10$), where an oversized majority reinforces correlated noise, leading to a diminished signal-to-noise ratio and collective fallacies.}
    \label{fig:agent_num}
\end{figure}

\section{Related Work}
Due to the improved reasoning ability of large language models~\cite{zhang2025viper,wang2026forest}, Multi-agent systems (MAS) based on LLMs have attracted significant research attention, with comprehensive surveys reviewing state-of-the-art approaches~\cite{guo2024largelanguagemodelbased, yan2025beyond}. 

\subsection{Multi-Agent Debate: Mechanisms and Diversity}

The core premise of MAD, that iterative argumentation can filter noise and improve reasoning has motivated numerous system designs across diverse tasks~\cite{liu2024skepticism,liu2025stepwise, liang2024encouraging,tang2024medagents, chen2024comm}. Early works focused on optimizing communication structures, with \citep{xiong2023examining} introducing enhancements grounded in debate theory and \citet{chanchateval} developing evaluation frameworks. Subsequent research has explored various architectural innovations, including dynamic communication protocols~\cite{liu2024dynamic, zhang2024cut} and efficient topology designs~\cite{li2024improving}.
A particularly relevant research direction emphasizes agent heterogeneity as a mechanism to escape echo chambers. \citet{chen2024reconcile} leverage heterogeneous LLM architectures, while \citet{liang2024encouraging, wang2024unleashing} inject distinct personas to induce divergent reasoning paths. \citet{liubreaking, chu2024exploring} further control generation diversity to prevent premature consensus, and learning-based approaches have been proposed to optimize debate dynamics~\cite{liu2024dynamic}. While prior work induces diversity heuristically, we provide the first theoretical characterization of why heterogeneity matters: it creates asymmetric cognitive potential energy that can be systematically exploited via peer prediction to break the Martingale Curse.

\subsection{Theoretical Limitations of Multi-Agent Debate}

Despite widespread adoption, recent analyses have revealed fundamental limitations of MAD. \citet{cemri2025multi} identified 14 distinct failure modes, while \citet{zhang2025if} demonstrated that MAD does not consistently outperform single-agent baselines. \citet{huang2024large} showed that LLMs lack sufficient self-correction capabilities, and \citet{smit2024should} found that MAD underperforms advanced single-agent reasoning methods. \citet{wang2024rethinking} further argued that well-prompted single agents can sometimes surpass MAD performance.
More critically, several works have documented convergence to incorrect consensus. \citet{xiong2023examining, zhang2025if} observed answer subversion during debate, and \citet{estornell2024multi} showed that MAD systems systematically converge to majority opinion even when that opinion reflects common misconception, precisely the Challenging Interval we formalize. \citet{benedikt2025voting} demonstrated that multiple debate rounds can decrease performance, and \citet{choi2025debate} provided a theoretical foundation by characterizing standard MAD as a martingale process: without external supervision, the expected belief in the correct answer remains constant across rounds.
We provide the first mechanism to break this Martingale Curse endogenously by converting asymmetric cognitive potential into submartingale drift.
\section{Conclusion}

MAD often fails due to the ``Martingale Curse," where correlated noise causes agents to converge toward shared hallucinations. We propose \textbf{AceMAD}, a framework that breaks this curse by harnessing \textit{asymmetric cognitive potential energy}. By leveraging peer-prediction to reveal second-order beliefs, we identify truth-holders who accurately anticipate collective errors, a meta-cognitive superiority formalized as Blackwell dominance. Our nonlinear mechanism converts this latent potential energy into submartingale drift, guaranteeing that truth-holders’ influence dominates the aggregate belief even from initial minority positions. Beyond proving that peer-prediction is the essential driver for signal recovery, we uncover a scaling threshold where excessive group size reinforces collective fallacies. AceMAD provides a model-agnostic 
mechanism to elicit truth from noise, 
offering a new paradigm 
for MAD in challenging environments.

\bibliography{references}
\bibliographystyle{unsrtnat}

\newpage
\appendix
\onecolumn

\section{Algorithm}
\label{app:alg}

We provide the complete algorithmic specification of AceMAD in Algorithm~\ref{alg:AceMAD}. The protocol operationalizes the theoretical framework developed in Section~\ref{sec:theory} through a four-phase iterative process: (1) Argumentation, where agents generate natural language arguments; (2) Signal Extraction, where agents commit to both self-beliefs and peer-predictions; (3) Verification, where the Brier score quantifies cognitive potential; and (4) Non-linear Amplification, where exponential weight updates convert this potential into submartingale drift. The final decision aggregates beliefs using squared weights to further amplify the truth-holder's influence.

\begin{algorithm}[htb]
   \caption{AceMAD}
   \label{alg:AceMAD}
\begin{algorithmic}[1]
   \STATE {\bfseries Input:} Query $x$, Set of Agents $\mathcal{A}$, Rounds $T$, Rate $\eta$
   \STATE {\bfseries Initialize:} Weights $w_i^{(0)} \leftarrow 1$ for all $i \in \{1, \dots, N\}$
   \STATE {\bfseries Initialize:} History $H^{(0)} \leftarrow \emptyset$
   \FOR{$t=1$ {\bfseries to} $T$}
       \FOR{each agent $i \in \mathcal{A}$ in parallel}
           \STATE Generate argument $m_i^{(t)}$ based on $H^{(t-1)}$
           \STATE Commit Self-Belief $p_i^{(t)}$ and Peer-Forecast $\hat{q}_i^{(t)}$
       \ENDFOR
       \STATE Update History $H^{(t)} \leftarrow H^{(t-1)} \cup \{m_1^{(t)}, \dots, m_N^{(t)}\}$
       \FOR{each agent $i \in \mathcal{A}$}
           \STATE Calculate peer realization $\bar{Q}_{-i}^{(t)} \leftarrow \frac{1}{N-1} \sum_{j \neq i} p_j^{(t)}$
           \STATE Compute score $S_i^{(t)} \leftarrow 1 - ||\hat{q}_i^{(t)} - \bar{Q}_{-i}^{(t)}||_2^2$
           \STATE Update weight $w_i^{(t)} \leftarrow w_i^{(t-1)} \cdot \exp(\eta \cdot S_i^{(t)})$
       \ENDFOR
       \STATE Normalize weights $w^{(t)} \leftarrow w^{(t)} / \sum_j w_j^{(t)}$
   \ENDFOR
   \STATE {\bfseries Return} $\arg\max_y \sum_i (w_i^{(T)})^2 p_i^{(T)}(y)$
\end{algorithmic}
\end{algorithm}

\section{Experimental Details}
\label{app:experiment}

\subsection{Implementation}
All experiments are conducted using the same base infrastructure to ensure fair comparison. We implement AceMAD using the following configuration:

(1)~\textbf{Model Selection}: We primarily use \texttt{GPT-4o-mini} as our base model, with additional validation on \texttt{Qwen3-235B-A22B-Instruct}, \texttt{DeepSeek-V3.1}, and \texttt{Llama-3.1-8B-Instruct} to verify generalization across architectures.

(2)~\textbf{Hyperparameters}: The amplification rate $\eta$ is set to 2.0 across all experiments unless otherwise specified. We found this value provides sufficient signal amplification while maintaining numerical stability.

(3)~\textbf{Agent Configuration}: Following Section~\ref{sec:heterogeneity_principle}, we induce heterogeneity through a dual-strategy approach in main experiment:
\begin{itemize}
    \item Crowd agents (80\% of the group): Temperature = 0.1, configured to approximate maximum likelihood behavior
    \item Truth-holder agents (20\% of the group): Temperature = 0.6, assigned critical-thinking personas to explore distributional tails
\end{itemize}

\subsection{Challenging Subset Construction}
For each benchmark, we construct the \textit{Challenging Interval subset}\footnote{We will release our challenging subsets for all benchmarks.} by filtering instances where a naive single-agent baseline (Temperature = 0.7) consistently fails. Specifically:

1. Run the single-agent baseline 5 times on each instance.

2. Select instances where accuracy $<$40\% across runs.

3. Manually verify that failures are due to systematic misconceptions rather than random errors.

4. The resulting subsets is the challenging subset.

This construction ensures we evaluate debate mechanisms precisely in the Challenging Interval where they should prove most valuable.

\subsection{Evaluation Metrics}
We report \textit{accuracy} as the primary metric, computed as the proportion of instances where the final aggregated decision matches the ground truth. For MAD variants, the final decision is determined by weighted voting (Eq.~\ref{prediction}). For majority voting, we use simple plurality among initial responses.

\section{Benchmarks}
\label{app:benchmark}
We evaluate AceMAD across six diverse benchmarks, each targeting different aspects of challenging reasoning. These benchmarks are specifically chosen because they expose common failure modes in LLMs: systematic biases, logical traps, and domain-specific misconceptions. Below we provide detailed descriptions of each benchmark and explain why they are particularly suitable for evaluating debate mechanisms in the Challenging Interval:

\noindent \textbf{TruthfulQA}~\cite{lin2022truthfulqa} is a benchmark designed to evaluate the truthfulness of LLMs, comprising 817 questions across 38 categories such as health, law, and finance. It is specifically constructed to induce models into generating false answers that mimic common human misconceptions, thereby assessing the model's ability to avoid hallucinations.

\noindent \textbf{ARC-Challenge}~\cite{clark2018think} focuses on grade-school science questions, with the Challenge subset strictly containing problems that retrieval-based or co-occurrence algorithms fail to answer. It requires models to possess not only scientific common sense but also deep reasoning capabilities to understand underlying physical processes.

\noindent \textbf{BBH}~\cite{suzgun2023challenging} (BIG-bench Hard) is a curated suite of 23 challenging tasks selected from BIG-bench where LLMs typically perform below the average human rater. It is designed to stress-test advanced reasoning abilities across diverse settings, including multi-step reasoning, symbolic manipulation, and complex logical structures.We chose the logic-related subset.

\noindent \textbf{LogiQA}~\cite{liu2021logiqa}, sourced from the National Civil Servants Examination of China, contains truth-holder-level logical reasoning and reading comprehension problems. It covers various forms of reasoning, including deductive and inductive logic, requiring models to extract logical relationships from complex texts for deep analysis.

\noindent \textbf{MedQA}~\cite{jin2021disease} is a large-scale biomedical question-answering dataset collected from medical licensing exams in the US (USMLE), Mainland China, and Taiwan. It benchmarks the model's proficiency in broad medical knowledge, clinical case analysis, and diagnostic reasoning within medical contexts.

\noindent \textbf{MMLU-Pro (Law)}~\cite{wang2024mmlu} is a robust version of the classic MMLU benchmark, designed to rigorously evaluate reasoning by increasing difficulty, removing noise, and expanding the option set. The Law subset specifically tests the model's mastery of complex legal provisions, case analysis, and jurisprudential logic.

\section{Baselines}
\label{app:baselines}
We compare AceMAD against four baseline approaches that represent different paradigms in Multi-Agent Debate. These baselines span the spectrum from zero-interaction ensemble methods (Majority Voting) to various MAD architectures with different communication topologies. Below we provide precise specifications of each baseline to ensure reproducibility and fair comparison:

\noindent \textbf{Decentralized MAD}~\cite{du2023improvingfactualityreasoninglanguage} adopts a fully connected communication topology. In each round of debate, every agent receives the complete history of responses from all other peer agents to update their own reasoning and answers, facilitating comprehensive information exchange across the entire collective.

\noindent \textbf{Sparse MAD}~\cite{li2024improving} is a variant of the decentralized framework designed to enhance computational efficiency. It utilizes a sparse communication topology where each agent interacts with only a limited subset of peers rather than the entire group, significantly reducing the information exchange overhead while maintaining collaborative reasoning.

\noindent \textbf{Centralized MAD}~\cite{guo2024largelanguagemodelbased} employs a star topology with a designated central agent or moderator. In this setup, the central agent is responsible for aggregating responses from all peer agents and generating a unified updated answer or summary at each debate round, acting as the primary bottleneck for information synthesis.

\noindent \textbf{Majority Voting} serves as a non-debate baseline that aggregates the initial zero-shot responses from multiple independent agents. The final answer is selected based on the most frequent prediction (consensus), which allows us to isolate the performance gains attributed specifically to the iterative multi-agent debate mechanism.

\section{Theoretical Analysis of AceMAD}
\label{app:theory}

\subsection{Setup and Challenging Interval Formulation}

We consider a single binary prediction task with ground truth
\[
Y \in \{0,1\},\qquad \mathbb{P}(Y=1)=1.
\]
There are \(N\) agents. One of them is an truth-holder \(E\) who possesses asymmetric cognitive potential energy, and the remaining \(N-1\) agents form the hallucinating majority \(C\) suffering from common misconceptions.

At round \(t\), each agent \(i\in\{1,\dots,N\}\) holds a subjective belief
\[
p_i^{(t)} := P_i^{(t)}(Y=1) \in [0,1].
\]
In the correlated noise regime, we assume that the truth-holder is systematically closer to the truth than the crowd:
\[
p_E^{(t)} \approx 1-\delta,\qquad
p_c^{(t)} \approx \varepsilon\quad\text{for all } c\in C,
\]
with \(\delta,\varepsilon\in(0,1/2)\), where $\delta$ represents the truth-holder's small error margin and $\varepsilon$ represents the crowd's persistent bias toward the incorrect answer.

In each round, the AceMAD protocol proceeds as follows:
\begin{enumerate}
    \item \textbf{Argue.} Agents produce natural-language arguments based on the debate history.
    \item \textbf{Commit.} Each agent privately commits to
    \begin{itemize}
        \item its own belief \(P_i^{(t)}\) over answers, and
        \item a \emph{peer prediction} \(\hat{Q}_i^{(t)}\), which is a prediction of the empirical mean of the other agents' reported beliefs.
    \end{itemize}
    \item \textbf{Reveal \& score.} All commitments are revealed. For each agent \(i\), we compute the realized average of others' reports,
    \[
    \bar{Q}_{-i}^{(t)} := \frac{1}{N-1}\sum_{j\neq i} P_j^{(t)},
    \]
    and assign a Brier-type peer prediction score
    \[
    S_i^{(t)} := 1 - \bigl\|\hat{Q}_i^{(t)} - \bar{Q}_{-i}^{(t)}\bigr\|_2^2.
    \]
    \item \textbf{Update.} We maintain non-negative weights \(\{w_i^{(t)}\}_{i=1}^N\). At the end of round \(t\), weights are updated multiplicatively:
    \[
    w_i^{(t+1)} = w_i^{(t)} \cdot \exp\bigl(\eta S_i^{(t)}\bigr),
    \]
    where \(\eta>0\) is a learning-rate parameter.
    \item \textbf{Aggregate.} The aggregated belief at round \(t\) is the weight-averaged belief
    \[
    p_w^{(t)}(Y=1) := 
    \frac{\sum_{i=1}^N w_i^{(t)}\,p_i^{(t)}}{\sum_{j=1}^N w_j^{(t)}}.
    \]
\end{enumerate}

We show that:
\begin{enumerate}
    \item The AceMAD channel Blackwell-dominates standard debate (Theorem~\ref{thm:blackwell});
    \item The truth-holder achieves a strictly higher expected peer-prediction score than the crowd (Theorem~\ref{thm:proper});
    \item Under multiplicative weight updates, the aggregated belief \(p_w^{(t)}(Y=1)\) forms a sub-martingale with a positive drift towards the truth (Theorem~\ref{thm:drift}).
\end{enumerate}

\subsection{Information-Theoretic Manifestation}

We view both standard debate and AceMAD as information channels from the ground truth \(Y\) to observable transcripts.

Let \(\sigma_{\mathrm{std}}\) denote the channel corresponding to standard multi-agent debate. On each instance, it outputs the transcript
\[
Z_{\mathrm{std}}
  = \bigl(A_i^{(t)}, P_i^{(t)}\bigr)_{i=1}^N,
\]
where \(A_i^{(t)}\) are the arguments and \(P_i^{(t)}\) the self-beliefs.

Let \(\sigma_{\mathrm{info}}\) denote the AceMAD channel. Its transcript is
\[
Z_{\mathrm{info}}
  = \bigl(A_i^{(t)}, P_i^{(t)}, \hat{Q}_i^{(t)}, S_i^{(t)}\bigr)_{i=1}^N,
\]
which strictly contains the standard transcript plus peer forecasts and scores.

We recall the Blackwell order on channels: for two channels \(\sigma_1\) and \(\sigma_2\) from \(Y\) to outputs \(Z_1\) and \(Z_2\), we say that \(\sigma_2\) Blackwell-dominates \(\sigma_1\) if there exists a stochastic kernel \(\kappa\) such that
\[
\sigma_1 = \kappa \circ \sigma_2.
\]

\begin{theorem}[Blackwell dominance of AceMAD]
\label{thm:blackwell_proof}
There exists a stochastic kernel \(\kappa\) such that
\[
\sigma_{\mathrm{std}} = \kappa \circ \sigma_{\mathrm{info}}.
\]
Hence, AceMAD Blackwell-dominates standard debate. Moreover, in the correlated noise regime, this dominance is strict.
\end{theorem}

\begin{proof}
Define the projection
\[
f: Z_{\mathrm{info}} \mapsto Z_{\mathrm{std}},\qquad
f\bigl((A_i,P_i,\hat{Q}_i,S_i)_{i=1}^N\bigr)
  := (A_i,P_i)_{i=1}^N.
\]
This projection simply discards peer forecasts and scores. We can turn \(f\) into a deterministic stochastic kernel \(\kappa\) by setting
\[
\kappa(z_{\mathrm{std}} \mid z_{\mathrm{info}}) :=
\begin{cases}
1, & \text{if } z_{\mathrm{std}} = f(z_{\mathrm{info}}),\\
0, & \text{otherwise.}
\end{cases}
\]
By construction, sampling \(Z_{\mathrm{info}} \sim \sigma_{\mathrm{info}}(\cdot \mid Y)\) and then applying \(\kappa\) yields exactly the same distribution over \(Z_{\mathrm{std}}\) as sampling directly from \(\sigma_{\mathrm{std}}(\cdot \mid Y)\). Hence \(\sigma_{\mathrm{std}} = \kappa \circ \sigma_{\mathrm{info}}\), which shows that AceMAD Blackwell-dominates standard debate in the weak sense.

To see that the dominance is strict in the correlated noise regime, consider the decision problem of selecting a single agent whose prediction will be used as the final report. Under AceMAD, the scores \(\{S_i^{(t)}\}\) reveal which agent is systematically better at forecasting peers (Theorem~\ref{thm:proper}), allowing a decision-maker to allocate more weight to the truth-holder and achieve lower Bayes risk in predicting \(Y\). Under standard debate, the scores are not observable, and in the correlated noise regime arguments and self-beliefs alone do not suffice to reliably identify the truth-holder. Therefore, there exists at least one decision problem (namely, choosing which agent to follow) for which AceMAD yields a strictly lower expected loss than standard debate. This implies strict Blackwell dominance.
\end{proof}

\subsection{Cognitive Potential Energy Separation }

We now show that, under the correlated noise assumptions and a strictly proper scoring rule, the truth-holder obtains a strictly higher expected peer-prediction score than the crowd.

We model the realized average of other agents' reports as a random probability vector. Fix an agent \(i\), and let
\[
X := \bar{Q}_{-i}^{(t)} \in \Delta(\mathcal{A})
\]
denote the empirical distribution over a finite answer set \(\mathcal{A}\). An agent reports a forecast \(q \in \Delta(\mathcal{A})\), which is scored by the Brier score
\[
S(q,X) := 1 - \|q - X\|_2^2
      = 1 - \sum_{a\in\mathcal{A}} (q_a - X_a)^2.
\]

We first recall a standard decomposition of squared error.

\begin{lemma}[Brier decomposition]
\label{lem:brier}
Let \(X\) be a random vector in \(\mathbb{R}^d\) with mean \(\mu := \mathbb{E}[X]\), and let \(q\in\mathbb{R}^d\). Then
\[
\mathbb{E}\bigl[\|q - X\|_2^2\bigr]
= \mathbb{E}\bigl[\|X - \mu\|_2^2\bigr] + \|q - \mu\|_2^2.
\]
\end{lemma}

\begin{proof}
Write \(\|q-X\|_2^2 = \sum_{k=1}^d (q_k - X_k)^2\). For each coordinate,
\[
\begin{aligned}
\mathbb{E}[(q_k - X_k)^2]
&= \mathbb{E}\bigl[(X_k - \mu_k + \mu_k - q_k)^2\bigr]\\
&= \mathbb{E}\bigl[(X_k - \mu_k)^2\bigr]
 + 2(\mu_k - q_k)\,\mathbb{E}[X_k - \mu_k]
 + (\mu_k - q_k)^2\\
&= \mathbb{E}\bigl[(X_k - \mu_k)^2\bigr] + (\mu_k - q_k)^2,
\end{aligned}
\]
since \(\mathbb{E}[X_k - \mu_k]=0\). Summing over all coordinates yields
\[
\mathbb{E}\bigl[\|q - X\|_2^2\bigr]
= \mathbb{E}\bigl[\|X - \mu\|_2^2\bigr] + \|q - \mu\|_2^2.
\]
\end{proof}

This shows that, in expectation, the Brier score is uniquely maximized by reporting the true mean \(\mu\).

We now apply this to the truth-holder and the crowd.

\begin{theorem}[Identification via proper scoring]
\label{thm:proper}
Let \(X := \bar{Q}_{-i}^{(t)}\) denote the random average of other agents' reports, with mean \(\mu := \mathbb{E}[X]\). Assume the truth-holder \(E\) has a correct probabilistic model of \(X\) and therefore reports \(\hat{Q}_E^{(t)} = \mu\) as its peer-forecast, while any crowd member \(c\in C\) reports \(\hat{Q}_c^{(t)}\) with \(\hat{Q}_c^{(t)} \neq \mu\) with positive probability. Then, for the Brier score,
\[
\mathbb{E}\bigl[S_E^{(t)}\bigr] > \mathbb{E}\bigl[S_c^{(t)}\bigr].
\]
\end{theorem}

\begin{proof}
Recall that \(S(q,X)=1-\|q-X\|_2^2\). For the truth-holder, Lemma~\ref{lem:brier} with \(q=\mu\) gives
\[
\mathbb{E}\bigl[\|\hat{Q}_E^{(t)} - X\|_2^2\bigr]
= \mathbb{E}\bigl[\|X - \mu\|_2^2\bigr].
\]
For a crowd member \(c\),
\[
\mathbb{E}\bigl[\|\hat{Q}_c^{(t)} - X\|_2^2\bigr]
= \mathbb{E}\bigl[\|X - \mu\|_2^2\bigr]
 + \bigl\|\hat{Q}_c^{(t)} - \mu\bigr\|_2^2,
\]
again by Lemma~\ref{lem:brier}. Since \(\hat{Q}_c^{(t)}\neq\mu\) with positive probability, we have \(\mathbb{E}\bigl[\|\hat{Q}_c^{(t)} - \mu\bigr\|_2^2\bigr] > 0\), and thus
\[
\mathbb{E}\bigl[\|\hat{Q}_c^{(t)} - X\|_2^2\bigr]
> \mathbb{E}\bigl[\|\hat{Q}_E^{(t)} - X\|_2^2\bigr].
\]

Translating back to scores,
\[
\begin{aligned}
\mathbb{E}\bigl[S_E^{(t)}\bigr]
&= 1 - \mathbb{E}\bigl[\|\hat{Q}_E^{(t)} - X\|_2^2\bigr],\\
\mathbb{E}\bigl[S_c^{(t)}\bigr]
&= 1 - \mathbb{E}\bigl[\|\hat{Q}_c^{(t)} - X\|_2^2\bigr],
\end{aligned}
\]
so
\[
\mathbb{E}\bigl[S_E^{(t)}\bigr]
- \mathbb{E}\bigl[S_c^{(t)}\bigr]
= \mathbb{E}\bigl[\|\hat{Q}_c^{(t)} - X\|_2^2\bigr]
 - \mathbb{E}\bigl[\|\hat{Q}_E^{(t)} - X\|_2^2\bigr] > 0.
\]
\end{proof}

\subsection{Converting Potential Energy into Submartingale Drift}

We now show that, under multiplicative weight updates and the score separation of Theorem~\ref{thm:proper}, the aggregated belief \(p_w^{(t)}(Y=1)\) forms a sub-martingale with a positive drift towards the truth-holder's belief.

For clarity, we first analyze a reduced system with two meta-agents:
\begin{itemize}
    \item the truth-holder \(E\), with belief \(p_E^{(t)}\);
    \item a meta-crowd \(C\), aggregating all crowd agents into a single agent with belief \(p_C^{(t)}\).
\end{itemize}
The general case with \(N>2\) agents follows by grouping all crowd weights into a single meta-agent.

Let \(w_E^{(t)}\) and \(w_C^{(t)}\) be their weights at round \(t\). Define the truth-holder's weight share
\[
\alpha_E^{(t)} := \frac{w_E^{(t)}}{w_E^{(t)} + w_C^{(t)}},\qquad
\alpha_C^{(t)} := 1 - \alpha_E^{(t)}.
\]
The aggregate belief is
\[
p_w^{(t)} := p_w^{(t)}(Y=1)
  = \alpha_E^{(t)} p_E^{(t)} + \alpha_C^{(t)} p_C^{(t)}.
\]

Weights are updated according to
\[
w_i^{(t+1)} = w_i^{(t)} \exp\bigl(\eta S_i^{(t)}\bigr),\qquad i\in\{E,C\}.
\]
Hence
\[
\alpha_E^{(t+1)} 
 = \frac{\alpha_E^{(t)} \exp\bigl(\eta S_E^{(t)}\bigr)}
        {\alpha_E^{(t)} \exp\bigl(\eta S_E^{(t)}\bigr)
       + \alpha_C^{(t)} \exp\bigl(\eta S_C^{(t)}\bigr)}.
\]

We first observe that the truth-holder's weight share increases exactly when its score exceeds the crowd's.

\begin{lemma}[Monotonicity of weight shares]
\label{lem:alpha}
For any fixed \(\alpha_E\in(0,1)\),
\[
\alpha_E^{(t+1)} \ge \alpha_E^{(t)}
\quad\Longleftrightarrow\quad
S_E^{(t)} \ge S_C^{(t)}.
\]
\end{lemma}

\begin{proof}
Let \(\alpha_E := \alpha_E^{(t)}\) and \(\alpha_C := 1-\alpha_E\), and write \(S_E := S_E^{(t)}\), \(S_C:=S_C^{(t)}\). Then
\[
\alpha_E^{(t+1)}
= \frac{\alpha_E e^{\eta S_E}}
       {\alpha_E e^{\eta S_E} + \alpha_C e^{\eta S_C}}.
\]
The inequality \(\alpha_E^{(t+1)} \ge \alpha_E\) is equivalent to
\[
\frac{\alpha_E e^{\eta S_E}}
     {\alpha_E e^{\eta S_E} + \alpha_C e^{\eta S_C}}
\ge \alpha_E.
\]
Assuming \(\alpha_E>0\), this can be rearranged as
\[
e^{\eta S_E} \ge \alpha_E e^{\eta S_E} + \alpha_C e^{\eta S_C}
\quad\Longleftrightarrow\quad
\alpha_C e^{\eta S_E} \ge \alpha_C e^{\eta S_C}
\quad\Longleftrightarrow\quad
S_E \ge S_C,
\]
since the exponential is strictly increasing. The converse implication follows by the same algebra.
\end{proof}

The change in the aggregated belief can be expressed directly in terms of the change in \(\alpha_E\). In the simplest case where \(p_E^{(t)}\) and \(p_C^{(t)}\) are time-invariant, we have
\[
p_w^{(t+1)} - p_w^{(t)}
= \bigl(\alpha_E^{(t+1)} - \alpha_E^{(t)}\bigr) (p_E - p_C).
\]
Since the truth-holder is more accurate, we have \(p_E > p_C\), so the sign of \(p_w^{(t+1)} - p_w^{(t)}\) matches that of \(\alpha_E^{(t+1)} - \alpha_E^{(t)}\), which by Lemma~\ref{lem:alpha} matches the sign of \(S_E^{(t)} - S_C^{(t)}\).

We now combine this with the expected score gap.

\begin{assumption}[Uniform score gap]
\label{assump:gap}
There exists \(\gamma>0\) such that for all rounds \(t\),
\[
\mathbb{E}\bigl[S_E^{(t)} - S_C^{(t)} \,\big|\, \mathcal{F}_t\bigr] \ge \gamma,
\]
where \(\mathcal{F}_t\) denotes the sigma-algebra generated by the history up to round \(t\).
\end{assumption}

Assumption~\ref{assump:gap} is justified in our correlated noise model by Theorem~\ref{thm:proper}: the truth-holder is strictly better at forecasting the crowd's average reports, and the gap is uniform as long as the crowd retains a non-trivial weight share.

We can now state the main result.

\begin{theorem}[Breaking the Martingale Curse: positive drift]
\label{thm:drift}
Suppose Assumption~\ref{assump:gap} holds and \(p_E > p_C\). Then, for sufficiently small \(\eta>0\), the aggregated belief process \(\{p_w^{(t)}(Y=1)\}_{t\ge 0}\) is a sub-martingale with respect to \(\{\mathcal{F}_t\}\):
\[
\mathbb{E}\bigl[p_w^{(t+1)}(Y=1) \,\big|\, \mathcal{F}_t\bigr]
\ge p_w^{(t)}(Y=1).
\]
Moreover, as long as both truth-holder and crowd retain non-zero weight share, the conditional drift is strictly positive:
\[
\mathbb{E}\bigl[p_w^{(t+1)}(Y=1) - p_w^{(t)}(Y=1) \,\big|\, \mathcal{F}_t\bigr] > 0.
\]
\end{theorem}

\begin{proof}
We focus on the two-meta-agent reduction described above. Let us write \(\alpha_E := \alpha_E^{(t)}\), \(\alpha_C := 1-\alpha_E\), and define the score difference
\[
D^{(t)} := S_E^{(t)} - S_C^{(t)}.
\]
From the weight update, we can express \(\alpha_E^{(t+1)}\) as
\[
\alpha_E^{(t+1)}
= \frac{\alpha_E e^{\eta S_E^{(t)}}}
       {\alpha_E e^{\eta S_E^{(t)}} + \alpha_C e^{\eta S_C^{(t)}}}
= \frac{\alpha_E e^{\eta D^{(t)}}}
       {\alpha_E e^{\eta D^{(t)}} + \alpha_C}.
\]
For fixed \(\alpha_E\) and \(\alpha_C\), define
\[
g(d) := \frac{\alpha_E e^{\eta d}}{\alpha_E e^{\eta d} + \alpha_C},
\]
so that \(\alpha_E^{(t+1)} = g(D^{(t)})\) and \(\alpha_E = g(0)\). A Taylor expansion of \(g\) around \(d=0\) yields
\[
g(d) = g(0) + g'(0)\,d + O(d^2),
\]
with
\[
g'(0) = \alpha_E \alpha_C \eta.
\]
Since the scores are bounded, \(D^{(t)}\) is bounded, and thus \(O(d^2)\) can be controlled uniformly. Therefore, there exists a constant \(K>0\) such that
\[
\alpha_E^{(t+1)} - \alpha_E
\ge \alpha_E \alpha_C \eta D^{(t)} - K \eta^2
\]
for all rounds \(t\).

Assuming \(p_E\) and \(p_C\) are fixed across rounds, the change in the aggregated belief is
\[
p_w^{(t+1)} - p_w^{(t)}
= \bigl(\alpha_E^{(t+1)} - \alpha_E^{(t)}\bigr) (p_E - p_C).
\]
Taking conditional expectations and applying Assumption~\ref{assump:gap},
\[
\begin{aligned}
\mathbb{E}\bigl[p_w^{(t+1)} - p_w^{(t)} \,\big|\, \mathcal{F}_t\bigr]
&= (p_E - p_C)\,
   \mathbb{E}\bigl[\alpha_E^{(t+1)} - \alpha_E^{(t)} \,\big|\, \mathcal{F}_t\bigr]\\
&\ge (p_E - p_C)\,
   \bigl(\alpha_E^{(t)} \alpha_C^{(t)} \eta \,\mathbb{E}[D^{(t)} \mid \mathcal{F}_t]
         - K \eta^2\bigr)\\
&\ge (p_E - p_C)\,\bigl(\alpha_E^{(t)} \alpha_C^{(t)} \eta \gamma - K \eta^2\bigr),
\end{aligned}
\]
where \(\gamma>0\) is the uniform score gap in Assumption~\ref{assump:gap}.

Choosing \(\eta>0\) small enough, e.g.
\[
0 < \eta \le \frac{\alpha_E^{(t)} \alpha_C^{(t)} \gamma}{2K},
\]
ensures that the bracketed term is non-negative, and strictly positive whenever \(\alpha_E^{(t)}\alpha_C^{(t)}>0\). Therefore,
\[
\mathbb{E}\bigl[p_w^{(t+1)}(Y=1) \,\big|\, \mathcal{F}_t\bigr]
\ge p_w^{(t)}(Y=1),
\]
with strict inequality as long as both agents retain non-zero weight share. This establishes the sub-martingale property and the strictly positive drift away from the initial state.
\end{proof}

By standard results on multiplicative weights and truth-holder advice, such a persistent score advantage implies that the truth-holder's weight converges to one, and hence the aggregated belief \(p_w^{(t)}(Y=1)\) converges to the truth-holder's belief, which is itself close to the ground truth under the correlated noise assumptions.

\section{Robustness Across Model Architectures}
\label{app:deepseek}
To verify that AceMAD's performance gains are not artifacts of specific model architectures or training procedures, we conduct additional experiments using \texttt{DeepSeek-V3.1}, a recent open-source model with different architectural design and pre-training data distribution compared to \texttt{GPT-4o-mini} and \texttt{Qwen3}. The results in Table~\ref{tab:grouped_models_deepseek} demonstrate that our peer-prediction mechanism successfully extracts asymmetric cognitive potential energy across diverse model families. 

Notably, \texttt{DeepSeek-V3.1} achieves particularly strong baseline performance on several benchmarks (e.g., 75\% on ARC-C with Majority Voting), suggesting that this model family may have less severe correlated noise in its default generation patterns. Nevertheless, AceMAD still provides consistent improvements (average +6.36\%), with the most substantial gains observed on LogiQA (+6.87\%) and BBH (+5.25\%). This validates our theoretical claim that second-order belief elicitation is a model-agnostic mechanism for breaking the Martingale Curse, effective even when baseline error rates are relatively low.
The one exception is TruthfulQA,where AceMAD shows a slight degradation. We hypothesize this occurs because \texttt{DeepSeek-V3.1}'s strong factual knowledge reduces the presence of a ``hallucinating majority" on this benchmark, causing peer predictions to become less informative. This pattern reinforces our framework's scope: AceMAD is specifically designed for the Challenging Interval where correlated misconceptions dominate; in easier regimes where errors are independent, Decentralized MAD or even single-agent inference may suffice.

\begin{table*}[htbp]
\small
    \centering
    
    \caption{Generalization across models: Performance comparison using the \textit{DeepSeek-V3.1} model ($N=5$). \textit{AceMAD} consistently outperforms all baseline variants, demonstrating its robustness across different LLM architectures.}
    \label{tab:grouped_models_deepseek}

\resizebox{0.95\textwidth}{!}{
\begin{tabular}{l cccccc | c}
\toprule
\textbf{Methods} 
& \textbf{ARC-C} 
& \textbf{LogiQA} 
& \textbf{MMLU-Pro} 
& \textbf{TruthfulQA} 
& \textbf{MedQA} 
& \textbf{BBH} 
& \textbf{Average} \\
\midrule

\textit{Majority Voting}
& 75.00 & 22.19 & 31.80 & 52.47 & 43.23 & 52.47 & 46.19 \\

\textit{Decentralized MAD}
& \underline{78.70} & 27.50 & 35.78 & 60.99 & 56.45 & 60.99 & 53.40 \\

\textit{Centralized MAD}
& 75.93 & 46.25 & \underline{36.54} & \underline{\textbf{84.75}} & 56.77 & 62.78 & 60.50 \\

\textit{Sparse MAD}
& 75.00 & \underline{47.19} & 35.78 & 62.78 & \underline{57.10} & \underline{84.75} & \underline{60.43} \\

\rowcolor{lightblue}
\textit{\textbf{AceMAD}} (T=3)
& \textbf{80.56} \inc{1.86} & \textbf{54.06} \inc{6.87} & \textbf{38.50} \inc{1.96} & 78.03 \dec{6.72} & \textbf{60.00} \inc{2.90} & \textbf{90.00} \inc{5.25} & \textbf{66.86} \inc{6.36} \\

\bottomrule
\end{tabular}
}
\end{table*}

\section{Performance Analysis on 8B-Scale Models}
\label{app:8b_analysis}

In Table~\ref{tab:small}, we evaluate \textit{AceMAD} using \texttt{Llama-3.1-8B-Instruct} to investigate its efficacy on models with relatively smaller parameter scales. The results demonstrate that our peer-prediction mechanism remains effective even within the 8B-parameter regime, consistently outperforming standard debate and voting baselines. However, we observe that the performance gains on \texttt{Llama-3.1-8B-Instruct} are less pronounced compared to those observed with larger models like \texttt{GPT-4o-mini}. 

This discrepancy can be attributed to the inherent trade-off between model capacity and second-order reasoning: while the 8B model possesses sufficient knowledge to identify misconceptions, its ability to maintain a stable and accurate ``peer-prediction'' ($\hat{q}_i$) is naturally more constrained. The sparse truth signals are harder to elicit when the base agents have higher baseline hallucination rates, yet \textit{AceMAD} still successfully induces a positive drift toward the ground truth. This confirms that the proposed framework provides robust benefits across the parameter spectrum, though the magnitude of improvement is partially modulated by the reasoning ceiling of the underlying model.

\begin{table}[htbp]
\small
\centering
\caption{Performance comparison on the \texttt{Llama-3.1-8B-Instruct} model ($N=5$). This table evaluates the effectiveness of \textit{AceMAD} using an open-source model backbone across all benchmarks, demonstrating that the peer-prediction mechanism remains robust across different model architectures.}
\label{tab:small}

\resizebox{0.5\textwidth}{!}{
\begin{tabular}{l cccc}
\toprule
\textbf{Methods} 
& \textbf{TruthfulQA} 
& \textbf{MedQA} 
& \textbf{BBH} 
& \textbf{Average} \\
\midrule

Single Agent
& 27.80 & \underline{28.06} & \underline{17.93} & 24.60 \\

\textit{Majority Voting}
& 30.49 & 27.42 & 17.59 & 25.17 \\

\textit{Decentralized MAD}
& 33.18 & 27.74 & 16.90 & 25.94 \\

\textit{Centralized MAD}
& \textbf{37.53} & 27.42 & 16.55 & \underline{27.17} \\

\textit{Sparse MAD}
& 33.63 & 26.45 & 16.55 & 25.54 \\

\rowcolor{lightblue}
\textit{\textbf{AceMAD}} (T=3)
& \underline{36.77} & \textbf{29.68} & \textbf{22.76} & \textbf{29.74} \\

\bottomrule
\end{tabular}
}
\end{table}

\section{Results on Full Datasets}
While the primary focus of our study is the \textit{Challenging Interval}, we also evaluate \textit{AceMAD} on the full versions of the ARC-C and LogiQA benchmarks to assess its performance in mixed-difficulty environments. As summarized in Table~\ref{tab:full}, our method continues to outperform Decentralized MAD protocols even when the majority of instances do not contain correlated misconceptions. In these ``easier" regimes, the peer-prediction mechanism remains stable and does not penalize correct majorities, as truth-holders and the crowd naturally converge on their forecasts. The consistent gains observed in Table 4 demonstrate that \textit{AceMAD} is not merely a specialized tool for adversarial cases, but a robust reasoning framework that maintains its advantage across the entire difficulty spectrum of real-world datasets.

\begin{table*}[htbp]
\small
\renewcommand{\arraystretch}{0.9} 
\centering

\caption{Performance comparison on the full benchmarks ($N=5$). Unlike the focus on the \textit{Challenging Interval}, this table reports results across the entire test sets of ARC-C and LogiQA. \textit{AceMAD} demonstrates consistent gains even in these mixed-difficulty regimes, with improvements over the baseline shown in parentheses.}
\label{tab:full}

\resizebox{0.45\textwidth}{!}{
\begin{tabular}{l cc | c}
\toprule
\textbf{Methods} 
& \textbf{ARC-C} 
& \textbf{LogiQA}  
& \textbf{Average} \\
\midrule

MAD
& 81.06 & 49.15 & 65.11 \\

\rowcolor{lightblue}
\textbf{AceMAD}
& \textbf{85.07} \inc{4.01} & \textbf{53.41} \inc{4.26} & \textbf{69.24} \inc{4.13} \\

\bottomrule
\end{tabular}
}
\end{table*}

\section{Prompt}
\label{app:prompt}

\begin{tcolorbox}[colframe=gray!20!white, colback=gray!3!white, coltitle=black, fonttitle=\bfseries, title=Persona Prompt]

"Generalist": "You are a helpful assistant. You trust common knowledge and consensus."

"Skeptic": "You are a strict skeptic. You actively look for common misconceptions and logical traps. You suspect the majority might be wrong."
\end{tcolorbox}

\begin{tcolorbox}[colframe=gray!20!white, colback=gray!3!white, coltitle=black, fonttitle=\bfseries, title=Argumentation Prompt]

Question: \{question\}

Options:
\{options\_str\}

Conversation History:
\{history\_context\}

Please provide a concise argument for what you believe is the correct answer.
Do NOT output JSON here. Just natural language debate.
\end{tcolorbox}

\begin{tcolorbox}[colframe=gray!20!white, colback=gray!3!white, coltitle=black, fonttitle=\bfseries, title=Commit Prompt]

Question: \{question\}
Options:
\{options\_str\}

Conversation History:
\{history\_contex\}

Task:

1. Self\_Prob: Assign probability (0.0-1.0) to options based on your belief.

2. Peer\_Prediction: Predict the AVERAGE probability distribution of the OTHER agents in this conversation.
(Critically: If others are 'Generalists', they might fall for common misconceptions. Predict their likely errors.)

Output JSON:
\{\{
    "self\_prob": \{\{"A": 0.1, ...\}\},
    "peer\_prediction": \{\{"A": 0.3, ...\}\}
\}\}
\end{tcolorbox}

\end{document}